\documentclass{article}

 \PassOptionsToPackage{numbers, compress}{natbib}

\usepackage[final]{neurips_2018}

\usepackage[utf8]{inputenc} 
\usepackage[T1]{fontenc}    
\usepackage{hyperref}       
\usepackage{url}            
\usepackage{booktabs}       
\usepackage{amsfonts}       
\usepackage{nicefrac}       
\usepackage{microtype}      

\usepackage{url}
\usepackage[english]{babel}
\usepackage{amsmath,amsthm}
\usepackage{amssymb,mathrsfs}
\usepackage{amsfonts}
\usepackage[final]{graphicx}
\usepackage{epstopdf}
\usepackage{tikz}
\usepackage{pgfplots}
\usepackage{subfig}
\usepackage{enumerate}
\usepackage{bbm}
\usepackage{verbatim}
\usepackage{float}

\usepackage{tikz}
\usepackage{color,soul}
\usetikzlibrary{arrows,positioning,fit,backgrounds}
\usetikzlibrary{calc}

\usetikzlibrary{arrows,positioning,fit,backgrounds,shapes}
\usetikzlibrary{calc}

\newtheorem{definition}{Definition}
\newtheorem{remark}{Remark}
\newtheorem{theorem}{Theorem}

\newtheorem{lemma}{Lemma}
\newtheorem{proposition}{Proposition}
\newtheorem{corollary}{Corollary}

\newtheorem{example}{Example}

\newcommand{\be}{\begin{equation}}
\newcommand{\ee}{\end{equation}}
\newcommand{\ben}{\begin{enumerate}}
\newcommand{\een}{\end{enumerate}}
\newcommand{\bea}{\begin{eqnarray}}
\newcommand{\eea}{\end{eqnarray}}
\newcommand{\bean}{\begin{eqnarray*}}
\newcommand{\eean}{\end{eqnarray*}}
\newcommand{\ex}{\mathbb{E}}
\newcommand{\pr}{\mathbb{P}}

\newcommand{\ps}{{\psi^*}^{-1}}
\newcommand{\Lt}{\Lambda_{X_t}(\lambda)}
\newcommand{\hy}{\mathcal{H}}
\newcommand{\pa}{\mathcal{P}}

\title{Chaining Mutual Information and Tightening Generalization Bounds}

\author{\vspace{3mm}
  Amir R. Asadi$^1$\thanks{Corresponding author: \texttt{aasadi@princeton.edu}} ~~~~ Emmanuel Abbe$^{1,2}$ ~~~~ Sergio Verd\'{u} \\
  $^1$Princeton University ~~~~ $^2$EPFL
}

\begin{document}

\maketitle

\begin{abstract}
Bounding the generalization error of learning algorithms has a long history, which yet falls short in explaining various generalization successes including those of deep learning. Two important difficulties are (i) exploiting the dependencies between the hypotheses, (ii) exploiting the dependence between the algorithm's input and output. Progress on the first point was made with the chaining method, originating from the work of Kolmogorov, and used in the VC-dimension bound. More recently, progress on the second point was made with the mutual information method by Russo and Zou '15. Yet, these two methods are currently disjoint. In this paper, we introduce a technique to combine the chaining and mutual information methods, to obtain a generalization bound that is both algorithm-dependent and that exploits the dependencies between the hypotheses. We provide an example in which our bound significantly outperforms both the chaining and the mutual information bounds. As a corollary, we tighten Dudley's inequality when the learning algorithm chooses its output from a small subset of hypotheses with high probability.
\end{abstract}

\section{Introduction}\label{Introduction}
\subsection{Motivation}
Understanding the generalization phenomenon in machine learning has been a central question for many years and revived in recent years with the success and mystery of deep learning: why do neural nets generalize well, although they operate in a classically overparametrized setting? In particular, classical generalization bounds do not explain this phenomenon (see e.g. \cite{Zhang}, \cite{belkin2018understand}). Even simpler instances of successful machine learning problems and algorithms are not explained satisfactorily with current generalization bounds, e.g. \cite{belkin2018understand}. This paper aims at deriving tighter generalization bounds for learning algorithms by combining ideas from information theory and from high dimensional probability.

Generalization bounds have evolved throughout the years, starting from the basic union bound over the hypothesis set, the refined union bound, Rademacher complexity, chaining and VC-dimension \cite{bousquet2004introduction}, \cite{SSS}; and algorithm-dependent bounds such as PAC-Bayesian bounds \cite{McAllester}, uniform stability \cite{bousquet2002stability}, compression bounds \cite{compression}, and recently, the mutual information bound \cite{Russo}. 

We highlight two pitfalls among the key limitations of current bounds:

\paragraph{A. Ignoring the dependencies between the hypotheses.} Consider the following example (which we refer to as Example I): an algorithm observes $G^2=(G_1,G_2)$, where $G_1$ and $G_2$ are two independent standard normal random variables; the hypothesis set $\hy=\{h_t: t\in T\}$ consists of functions $h_t(G^2)\triangleq\langle t, G^2 \rangle $, where $T\triangleq\{t\in \mathbb{R}^2: \Vert t \Vert_2=1\}$. Suppose the algorithm is designed to choose the hypothesis which achieves $\max_{t\in T}h_t(G^2)$. Since $h_t(G^2), t\in T$ are all zero mean random variables, the expected bias of the algorithm is $\ex\left[\max_{t\in T}h_t(G^2)\right]$. Moreover, since $\hy$ consists of an uncountable number of hypotheses,  the union bound (or equivalently the maximal inequality) over the hypothesis set  
gives a vacuous bound. However, the fact is that we are not dealing with infinite number of \emph{independent} random variables: the random variables $h_t(G^2)$ and $h_s(G^2)$ are actually quite dependent on each other when $t$ and $s$ are close.

To exploit the dependencies, the powerful technique of \emph{chaining} has been developed in high dimensional probability in order to obtain uniform bounds on random processes, and has proven successful in a variety of problems including statistical learning. More specifically, chaining is the method for proving the tightest generalization bound using VC-dimension \cite{Ramon}, \cite{Vershynin}.  Originating from the work of Kolmogorov in 1934 (see \cite[p. 149]{Ramon}) and later developed by Dudley, Fernique, Talagrand and many others \cite{TalagrandBook}, the basic idea of chaining is to first describe the dependencies between the hypotheses by a metric $d$ on the set $T$, then to discretize $T$ and to approximate the maximal value ($\max_{t\in T}h_t(G^2)$) by approximating the maxima over successively refined finite discretizations, using union bounds at each step, and by introducing the notion of \emph{$\epsilon$-nets} and \emph{covering numbers} \cite{Lugosi}. For instance, with this method, one can prove the finite upper bound $\ex\left[\max_{t\in T}h_t(G^2)\right]\leq 19.0353$.  
 Even for many examples of finite hypothesis sets, chaining is known to give far tighter bounds than the union bound \cite{Ramon}. 
 Next we state a fundamental result which is based on the chaining method. For a metric space $(T,d)$, let $N(T,d,\epsilon)$ denote the covering number of $(T,d)$ at scale $\epsilon$. For the definitions of $\epsilon$-net and covering number, see Definition \ref{epsilonnet} in Section \ref{LipschitzProcSection} of the supplementary material, and for the definition of seperable subgaussian processes see Definitions \ref{SubgaussianProcessDefinition} and \ref{Separable process}.
 
\begin{theorem}[Dudley]\label{DudleyTheorem} \cite{Dudley}.
Assume that $\{X_t\}_{t\in T}$ is a separable subgaussian process on the bounded metric space $(T,d)$. Then
\begin{equation}\label{DudleyIneq}
\ex\left[\sup_{t\in T}X_t\right]\leq 6\sum_{k\in \mathbb{Z}}2^{-k}\sqrt{\log N(T,d,2^{-k})}.
\end{equation}
\end{theorem}
Note that PAC-Bayesian bounds, compression bounds and bounds based on uniform stability also do not exploit the dependencies between the hypotheses as they are not based on any metric on the hypothesis set.

\paragraph{B. Ignoring the dependence between the algorithm input (data) and output.} Generalization bounds based on Rademacher complexity\footnotemark and VC-dimension only depend on the hypothesis set and not on the algorithm, effectively rendering them too pessimistic for practical algorithms. Recent experimental findings in \cite{Zhang} have shown that in the over-parameterized regime of deep neural nets, such complexity measures give vacuous bounds for the generalization error. A possible explanation for that vacuousness is as follows: if $\mathcal{H}=\{h_t: t\in T\}$ denotes the hypothesis set and for every $t\in T$, $X_t$ denotes the generalization error of hypothesis $h_t$ and $W$ denotes the index of the chosen hypothesis by the algorithm, then to upper bound the expected generalization error $\ex[X_W]$, one uses 
\begin{equation}\label{ToyExampleInequality0}
\ex[X_W]\leq \ex\left[\sup_{t\in T}X_t\right],
\end{equation}
and aims at upper bounding $\ex\left[\sup_{t\in T}X_t\right]$ with these bounds, hence giving a \emph{uniform bound} over the generalization errors of the \emph{entire hypothesis set}. However, all we need to control is the generalization error of the \emph{specific hypothesis} $W$ selected by the algorithm. That expected generalization error of $W$ can be much smaller than the right side of (\ref{ToyExampleInequality0}) (see also \cite{Kawaguchi}). In other words, such bounds are not taking into account the input-output relation of the algorithm, and uniform bounding seems to be too stringent for these applications.  Consider the following example (which we refer to as Example II): let $X_1,X_2,...,X_n$ be standard normal random variables and assume that the algorithm output is index $W$. Therefore the expected bias of the algorithm is $\ex[X_W]$ and the goal is to upper bound it. By the maximal inequality (or equivalently the union bound), we have 
\footnotetext{Here we are referring to the Rademacher average of the entire hypothesis set. There exist other notions of Rademacher averages which are used in algorithm-dependent bounds, such as in local Rademacher complexities \cite{bartlett2005local}.}
\begin{equation}\label{ToyExampleInequality1}
\ex\left[\sup_{1\leq i\leq n}X_i\right]\leq \sqrt{2\log n},
\end{equation}
where (\ref{ToyExampleInequality1}) is asymptotically tight if $X_i, i=1,2,...,n$ are independent (see \cite[Chapter 2]{Lugosi}). But what if the algorithm is always more likely to choose $W$ among a small subset of $\{1,2,...,n\}$? Then  $\ex[X_W]$ could be much smaller than
the right side of  \eqref{ToyExampleInequality1}, as the chances of having an outlier value is smaller. Or, if the choice of $W$ is
not dependent on the data, then $\ex[X_W]=\ex[\ex[X_W|W]]=0$. Interestingly, to explain this phenomenon and to obtain tighter upper bounds on $\ex[X_W]$ an important information theoretic measure appears: the \emph{mutual information}. This was originally proposed  in the key paper of Russo and Zou \cite{Russo} and then generalized in \cite{Jiao}, \cite{Jiao2},  and in \cite{Raginsky} for infinite number of hypotheses:

\begin{theorem}\label{Xu Raginsky Theorem}\cite{Russo}\cite{Raginsky}
Let $\{X_t\}_{t\in T}$ be a random process and $T$ an arbitrary set. Assume that $X_t$ is $\sigma^2$-subgaussian and $\ex[X_t]=0$ for every $t\in T$, and let $W$ be a random variable taking values on $T$. Then 
\begin{equation}\label{Xu Raginsky Inequality}
|\ex[X_W]|\leq \sqrt{2\sigma^2 I(W;\{X_t\}_{t\in T})}.
\end{equation} 
\end{theorem}

In Example II, instead of using (\ref{ToyExampleInequality0}) and (\ref{ToyExampleInequality1}), one can have the tighter upper bound
\begin{equation}\label{Toy Example Inequality 2}
\ex[X_W]\leq \sqrt{2I(W;X_1,...,X_n)}.
\end{equation}
For example, if the algorithm chooses $W$ among $\{1,2,...,\lceil\log n \rceil\}$ with probability $1-o(1)$, then (\ref{Toy Example Inequality 2}) implies 
\begin{equation}
\ex[X_W]\leq \sqrt{2\left((1-o(1)\right)\log (\log n)+o(1)\log (n-\log n)+1}\ll \sqrt{2\log n}.
\end{equation}
However, this method does not give a finite bound for Example I, since 
\begin{equation}
I\left(\mathrm{argmax}_{t\in T}h_t(G^2);\{h_t(G^2)\}_{t\in T}\right)=\infty.
\end{equation}

Similarly, as discussed in \cite{PJL}, the mutual information bound for perturbed SGD or any iterative algorithm which adds degenerate noise in each iteration blows up, and information-theoretic strategies for analyzing generalization error of such algorithms have not been reported.

\subsection{This paper} 
By combining the ideas of the chaining method and the mutual information method, in this paper we obtain a \emph{chained mutual information} bound on the expected generalization error which takes into account the dependencies between the hypotheses as well as the dependence between output and input of the algorithm. When applied to the two aforementioned simple examples (Examples I and II), our bound yields the better bound between the classical chaining and classical mutual information bounds. More importantly,  we provide examples for which our bound outperforms both of the previous bounds significantly: in Example \ref{Running Example} we provide a family of cases where the chaining method gives a relatively large constant, the mutual information bound blows up, but our bound tends towards zero. We also discuss how our new  bound gives a possible direction to explain the phenomenon described in \cite{PJL} (see Remark \ref{PJL Remark}), and to exploit regularization properties of some algorithms (see Section \ref{Small set section}).

\subsection{Further related literature} 
In \cite{Moran}, the mutual information between the input and the output of binary classification learning algorithms is used to obtain high probability generalization bounds.

PAC-Bayesian bounds are another type of algorithm-dependent bounds which are concerned with finding high probability generalization bounds for randomized classifiers \cite{McAllester}. These bounds define a hierarchy over the hypothesis set by using a prior distribution on that set \cite{SSS}.  As discussed in \cite{Moran}, there is a connection and similarity between PAC-Bayesian bounds and the mutual information bound, both using the variational representation of relative entropy in their proofs. In \cite{audibert2004pac} and \cite{audibert2007combining}, the authors combine the ideas of PAC-Bayesian bounds with generic chaining and create high probability bounds for randomized classifiers. Their use of an auxiliary sample set and the notion of average distance between partitions makes their bounds conceptually different from our work. However, their bounds have the advantage to exploit the variance of the hypotheses and to give high probability results. 

 In the probability theory literature, Fernique \cite{Fernique} gives upper and lower bounds on the expected bias of an algorithm (or a selection rule) which chooses its output from a Gaussian process, by using a chaining argument while taking into account the marginal distribution of the algorithm output. We further utilize the dependence between the algorithm input and output and the stochasticity of the algorithm, and we give results for more general processes. However, we only obtain upper bounds in this paper. 

\subsection{Notation}
In the framework of supervised statistical learning, $\mathcal{X}$ is the instances domain, $\mathcal{Y}$ is the labels domain and $\mathsf{Z}=\mathcal{X}\times \mathcal{Y}$ denotes the examples domain. Furthermore, $\mathcal{H}=\{h_w : w\in \mathcal{W}\}$ is the hypothesis set where the hypotheses are indexed by an index set $\mathcal{W}$, and there is a nonnegative loss function $\ell:\mathcal{H}\times \mathsf{Z}\to \mathbb{R}^+$. A learning algorithm receives the training set $S=(Z_1,Z_2,...,Z_n)$ of $n$ examples with i.i.d. random elements drawn from $\mathsf{Z}$ with distribution $\mu$. Then it picks an element $h_W\in\mathcal{H}$ as the output hypothesis according to a random transformation $P_{W|S}$ (thus, we are allowing randomized algorithms). For any $w\in\mathcal{W}$, let 
\begin{equation}
L_{\mu}(w)\triangleq \ex[\ell(h_w,Z)], \hspace{4mm} Z\sim \mu 
\end{equation}
denote the statistical (or population) risk of hypothesis $h_w$. For a given training set $S$, the empirical risk of hypothesis $h_w$ is defined as 
\begin{equation}
L_{S}(w)\triangleq \frac{1}{n}\sum_{i=1}^n \ell(h_w,Z_i),
\end{equation}
and the generalization error of hypothesis $h_w$ (dependent on the training set) is defined as
\begin{equation}
\mathrm{gen}(w)\triangleq L_{\mu}(w)-L_S(w).
\end{equation}
Averaging with respect to the joint distribution $P_{S,W}=\mu^{\otimes n}P_{W|S}$, we denote the expected generalization error and the expected absolute value of generalization error by
\begin{equation}
\mathrm{gen}(\mu, P_{W|S})\triangleq \ex [L_{\mu}(W)-L_S(W)],
\end{equation}
and
\begin{equation}
\mathrm{gen}^+(\mu, P_{W|S})\triangleq \ex [|L_{\mu}(W)-L_S(W)|],
\end{equation}
respectively. Our purpose is to  find upper bounds on $\mathrm{gen}(\mu, P_{W|S})$ and $\mathrm{gen}^+(\mu, P_{W|S})$. 

Let $X_{\mathcal{N}}\triangleq\{X_i: i\in\mathcal{N}\}$ denote a random process indexed by the elements of the set $\mathcal{N}$. Let $\mathbf{0}$ denote the identically zero function. In this paper, all logarithms are in natural base and all information theoretic measures are in nats. $H(X)$  denotes the Shannon entropy of a discrete random variable $X$,  and $h(Y)$ denotes the differential entropy of an absolutely continuous random variable $Y$.  
\section{Main results}\label{Main results}
Assume that $\{X_t\}_{t\in T}$ is a random process with index set $T$. In the chaining method, we impose a metric $d$ on $T$ which describes the dependencies between the random variables. The widely used \emph{subgaussian processes} capture this notion and they arise in many applications: 
\begin{definition} [Subgaussian process] \label{SubgaussianProcessDefinition} The random process $\{X_t\}_{t\in T}$ on the metric space $(T,d)$ is called \emph{subgaussian} if $\ex[X_t]=0$ for all $t\in T$ and 
\begin{equation}
\ex\left[e^{\lambda(X_t-X_s)}\right]\leq e^{\frac12  \lambda^2 d^2(t,s)} \textrm{ ~  for all ~  } t,s\in T,  \lambda\geq 0.
\end{equation} 
\end{definition}
For example, based on the Azuma--Hoeffding inequality, $\{\mathrm{gen}(w)\}_{w\in \mathcal{W}}$ is a subgaussian process with the metric 
\begin{equation}
d(\mathrm{gen}(w), \mathrm{gen}(v))\triangleq\frac{\Vert \ell(h_w,\cdot)-\ell(h_v,\cdot)\Vert_{\infty}}{\sqrt{n}},
\end{equation}  
regardless of the choice of distribution $\mu$ on $\mathsf{Z}$. 

The following is a technical assumption which holds in almost all cases of interest:
\begin{definition} [Separable process] \label{Separable process}
The random process $\{X_t\}_{t\in T}$ is called \emph{separable} if there is a countable set $T_0\subseteq T$ such that $X_t\in \lim_{\substack{s\rightarrow t \\ s\in T_0}} X_s~$ for all $~t\in T~$  a.s.,
where $x\in \lim_{\substack{s\rightarrow t \\ s\in T_0}} x_s$ means that there is a sequence $(s_n)$ in $T_0$ such that $s_n\rightarrow t$ and $x_{s_n}\rightarrow x$.
\end{definition}
For example, if $t\to X_t$ is continuous a.s., then $X_t$ is a separable process \cite{Ramon}. 

Our main results rely on the notion of increasing sequence of $\epsilon$-partitions of the metric space $(T,d)$:

\begin{definition}[Increasing sequence of $\epsilon$-partitions] We call a partition $\pa=\{A_1, A_2,...,A_m\}$ of the set $T$ an \emph{$\epsilon$-partition} of the metric space $(T,d)$ if for all $i=1,2,...,m$, $A_i$ can be contained within a ball of radius $\epsilon$. A sequence of partitions $\{\pa_k\}_{k=m}^{\infty}$ of a set $T$ is called an \emph{increasing sequence} if for all $k\geq m$ and each $A\in\pa_{k+1}$, there exists $B\in \pa_k$ such that $A\subseteq B$. For any such sequence and any $t\in T$, let $[t]_k$ denote the unique set $A\in \pa_k$ such that $t\in A$.
\end{definition}

Assume now that $(T,d)$ is a bounded metric space, and let $k_1(T)$ be an integer such that $2^{-(k_1(T)-1)}\geq \mathrm{diam}(T)$. We have the following upper bounds on $\mathrm{gen}(\mu, P_{W|S})$ and $\mathrm{gen^+}(\mu, P_{W|S})$  based on the mutual information between the training set $S$ and the discretized output of the learning algorithm, where each of these mutual information terms is multiplied by an exponentially decreasing weight $2^{-k}$, in which the exponent measures how finely the output $W$ of the learning algorithm is discretized: 

\begin{theorem}\label{Generalization with chaining MI} Assume that $\{\mathrm{gen}(w)\}_{w\in \mathcal{W}}$ is a separable subgaussian process on the bounded metric space $(\mathcal{W},d)$. Let $\{\mathcal{P}_k\}_{k=k_1(\mathcal{W})}^{\infty}$ be an increasing sequence of partitions of $\mathcal{W}$, where for each $k\geq k_1(\mathcal{W})$, $\mathcal{P}_k$ is a $2^{-k}$-partition of $(\mathcal{W},d)$. 
\begin{enumerate}[(a)]
\item 
\begin{align}
&\mathrm{gen}(\mu, P_{W|S})\leq 3\sqrt{2}\sum_{k=k_1(\mathcal{W})}^{\infty}2^{-k}\sqrt{I([W]_k;S)},
\end{align}
\item If $\mathbf{0}\in \{\ell(h_w,\cdot): w\in\mathcal{W}\}$, then
\begin{align}
&\mathrm{gen^+}(\mu, P_{W|S})\leq 3\sqrt{2}\sum_{k=k_1(\mathcal{W})}^{\infty}2^{-k}\sqrt{I([W]_k;S)+\log 2}.
\end{align}
\end{enumerate}
\end{theorem}

 \begin{remark}\normalfont
Based on the general definition of mutual information with partitions (\cite[p. 252]{Cover}), we have $I(W;S)=\sup_k I([W]_k;S)$ therefore $I([W]_k;S)\to I(W;S)$ as $k\to \infty$.
\end{remark}
Theorem \ref{Generalization with chaining MI} is stated in the context of statistical learning. The
more general counterpart in the context of random processes  is:
\begin{theorem}\label{Chaining MI Random Process} Assume that $\{X_t\}_{t\in T}$ is a separable subgaussian process on the bounded metric space $(T,d)$. Let $\{\mathcal{P}_k\}_{k=k_1(T)}^{\infty}$ be an increasing sequence of partitions of $T$, where for each $k\geq k_1(T)$, $\mathcal{P}_k$ is a $2^{-k}$-partition of $(T,d)$. 
\begin{enumerate}[(a)]
\item
\begin{align}
\ex [X_W]\leq 3\sqrt{2}\sum_{k=k_1(T)}^{\infty}2^{-k}\sqrt{I([W]_k;X_T)}.
\end{align}
\item For any arbitrary $t_0\in T$,
\begin{equation}
\ex[|X_W-X_{t_0}|]\leq 3\sqrt{2}\sum_{k=k_1(T)}^{\infty}2^{-k}\sqrt{I([W]_k;X_T)+\log 2}.
\end{equation}
\end{enumerate}
\end{theorem}

Note that in Theorem \ref{Chaining MI Random Process} if we let $T\triangleq \mathcal{W}$ and $X_w\triangleq \mathrm{gen}(w)$ for all $w\in \mathcal{W}$, then for each $k\geq k_1(T)$, due to the Markov chain 
\begin{equation}
X_T=\{\mathrm{gen}(w)\}_{w\in \mathcal{W}}\leftrightarrow S \leftrightarrow W \leftrightarrow [W]_k,
\end{equation}
and the data processing inequality, we have $I([W]_k;X_T)\leq I([W]_k;S)$. Therefore Theorem \ref{Generalization with chaining MI} follows from Theorem \ref{Chaining MI Random Process}. The proof of Theorem \ref{Chaining MI Random Process} and the etymology of ``chaining mutual information" is given in Section \ref{Proof outline}. 

\begin{remark}\normalfont
For random processes other than subgaussian processes, where the tail of increments are controlled by a function $\psi$, similar results can be derived from Theorem \ref{Chaining Ineq Psi Theorem} in Section \ref{Proof of Chaining MI Theorem} of the supplementary material. 
\end{remark}
Both Theorem \ref{Generalization with chaining MI} and Theorem \ref{Chaining MI Random Process} capture the dependencies between the hypotheses by utilizing a metric $d$, and they are algorithm-dependent as the mutual information between the algorithm's discretized output and its input appears in their bounds. Now, to demonstrate the power of Theorem \ref{Chaining MI Random Process} and to compare it with the existing results in the literature, consider the following example:

\begin{example}\label{Running Example}\normalfont
Let $T$ be an arbitrary subset of $\mathbb{R}^n$, and $G^n\triangleq (G_1,...,G_n)\sim \mathcal{N}(0,I_n)$ be a standard normal random vector in $\mathbb{R}^n$. The \emph{canonical Gaussian process} is defined as $\{X_t\}_{t\in T}$, where
\begin{equation}
X_t\triangleq \langle t,G^n \rangle \textrm{ ~for all~ } t\in T.
\end{equation}
Note that $\{X_t\}_{t\in T}$ is a subgaussian process on the metric space $(T,d)$, where $d$ is the Euclidean distance.

Consider a canonical Gaussian process where $n=2$ and $T\triangleq\{t\in \mathbb{R}^2: \Vert t \Vert_2=1\}$. The process $\{X_t\}_{t\in T}$ can be reparameterized according to the phase of each point $t\in T$: the random variable $X_t$ can also be denoted as $X_{\phi}$, where $\phi\in [0,2\pi)$ is the phase of $t$. In other words, $\phi$ is the unique number in $[0,2\pi)$ such that $t=(\sin \phi, \cos \phi)$. Henceforth, we will assume the indices are in the phase form.

Let the relation between the input $X_T$ of an algorithm and its output $W$ be as 
\begin{equation}
W\triangleq\left(\textrm{argmax}_{\phi\in [0,2\pi)} X_{\phi}\right)\oplus Z  \textrm{~~(mod~}2\pi\textrm{)},
\end{equation}
where the noise $Z$ is independent from $X_T$, and has an atom with probability mass $\epsilon$ on $0$, and $1-\epsilon$ probability is uniformly distributed on $(-\pi,\pi)$. Note that since $Z$ has a singular (degenerate) part, $h(Z)=-\infty$.  

Due to symmetry, $W$ has uniform distribution over $[0,2\pi)$. But we have
\begin{align}
I(W;X_{T})&=h(W)-h(W|X_{T})\\
					&=\log 2\pi- h\left(\textrm{argmax}_{\phi\in [0,2\pi)} X_{\phi}\oplus Z \middle|X_{T}\right)\\
					&=\log 2\pi- h(Z|X_{T})\\
					&=\log 2\pi - h(Z)\\
					&=\infty.
\end{align} 

Hence the upper bound on $\ex[X_W]$ due to the mutual information method (Theorem \ref{Xu Raginsky Theorem}) blows up:
\begin{align}
\ex[X_W]&\leq \sqrt{2I(W;X_T)} =\infty.
\end{align}

Note that $2^{-(-2)}\geq \mathrm{diam}(T)=2$. Therefore let $k_1(T)\leftarrow -1$ and for all integers $k\geq -1$, define
\begin{align}
\pa_k\triangleq \left\{\left[0, \frac{2\pi}{2^{k+2}}\right), \left[\frac{2\pi}{2^{k+2}},2\times\frac{2\pi}{2^{k+2}}\right),..., \left[\left(2^{k+2}-1\right)\frac{2\pi}{2^{k+2}}, 2\pi\right)\right\}.
\end{align}

It is clear that $\{\mathcal{P}_{k}\}_{k=-1}^{\infty}$ is an increasing sequence of partitions of $T$. Furthermore, for each $k\geq -1$, the length of the arc of each set in $\pa_k$ is $\delta_k\triangleq \frac{2\pi}{2^{k+2}}<2^{1-k}$. Thus each $\pa_k$ is a $2^{-k}$-partition of $(T,d)$ and $|\pa_k|=2^{k+2}$ (see Figure \ref{Partition figure}).
\begin{figure}[h]
  \centering
  \fbox{\includegraphics[width=9cm]{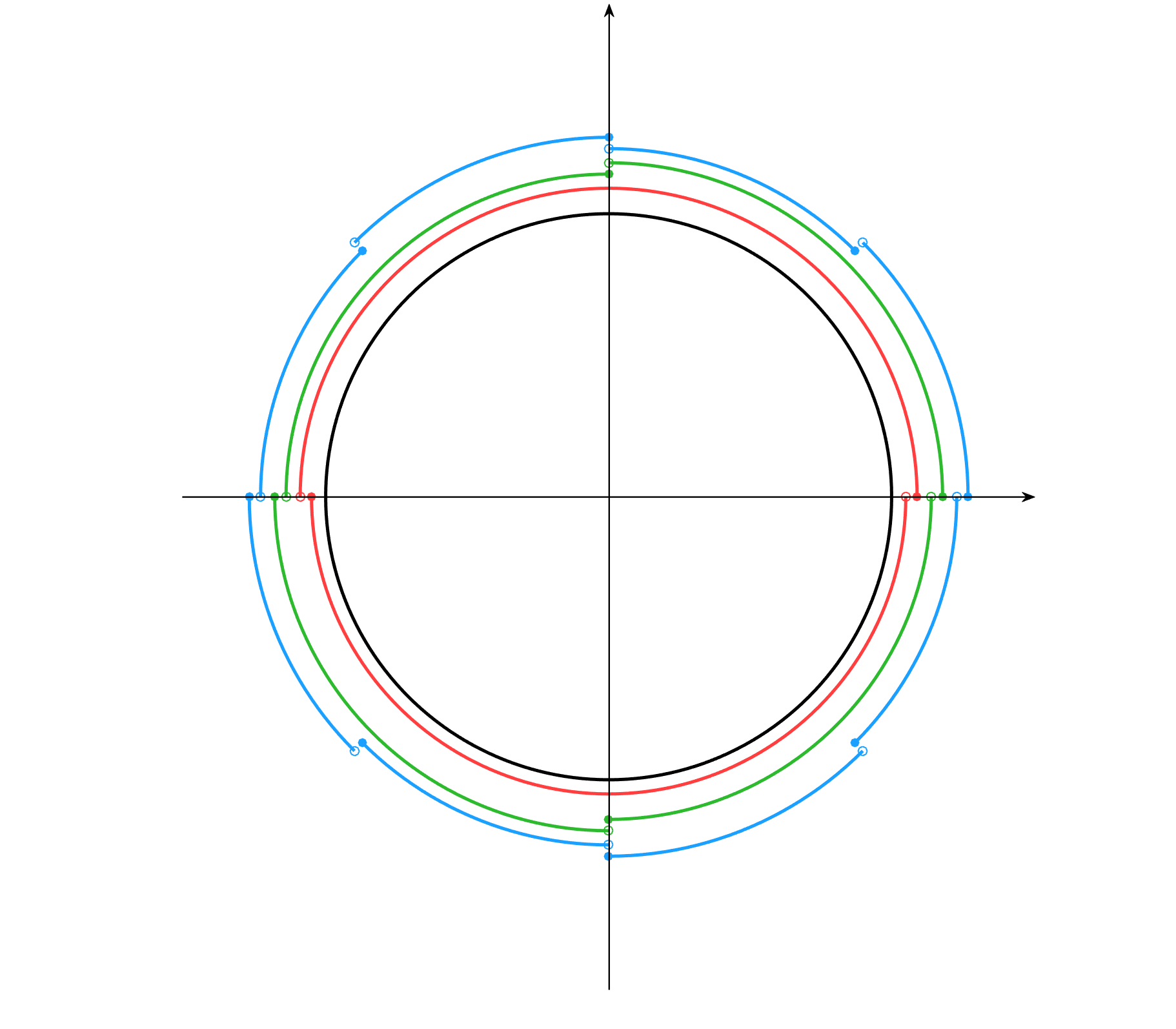}}{\rule[-.5cm]{0cm}{2cm} \rule[-.5cm]{7cm}{0cm}}
  \caption{Depiction of $T, \pa_{-1}, \pa_{0}$ and $\pa_{1}$ in the $\mathbb{R}^2$ plane. (The three partitions are magnified for clarity.)}
  \label{Partition figure}
\end{figure}

Now by using the classical chaining method (Theorem \ref{DudleyTheorem}) to upper bound $\ex[X_W]$ by upper bounding $\ex[\sup_{\phi\in[0,2\pi)}X_{\phi}]$ and ignoring the algorithm, we get 
\begin{align}
\ex [X_W]&\leq \ex\left[\sup_{\phi\in[0,2\pi)}X_{\phi}\right]\\
			   &\leq 3\sqrt{2}\sum_{k=-1}^{\infty}2^{-k}\sqrt{\log 2^{k+2}}\\
			   &=19.0352...\footnotemark
\end{align}
\footnotetext{The exact value of the bound in Theorem \ref{DudleyTheorem} is slightly smaller, since with our partitions we are using a rough approximate for the covering numbers. For example, at scale $2^{-(-1)}$, the covering number is $1$, while we have used partition $\mathcal{P}_{-1}$ with $\vert \mathcal{P}_{-1}\vert =2$ sets.}

On the other hand, for every $k\geq -1$ we have 
\begin{align}
I([W]_k;X_{T})&=H([W]_k)-H\left([W]_k|X_{T}\right)\\
					&=\log 2^{k+2}- H\left(\left[\left(\textrm{argmax}_{\phi\in [0,2\pi)} X_{\phi}\right)\oplus Z\right]_k \middle|X_{T}\right)\\
					&=\log 2^{k+2}- H\left(\epsilon+\frac{1-\epsilon}{2^{k+2}},\frac{1-\epsilon}{2^{k+2}},\dots,\frac{1-\epsilon}{2^{k+2}}\right).
\end{align} 

Therefore, based on the chained mutual information method (Theorem \ref{Chaining MI Random Process}), we have
\begin{align}
\ex [X_W]&\leq 3\sqrt{2}\sum_{k=-1}^{\infty}2^{-k}\sqrt{I([W]_k;X_T)}\\
			   &=3\sqrt{2}\sum_{k=-1}^{\infty}2^{-k}\sqrt{\log 2^{k+2}-H\left(\epsilon+\frac{1-\epsilon}{2^{k+2}},\frac{1-\epsilon}{2^{k+2}},\dots,\frac{1-\epsilon}{2^{k+2}}\right)}\label{CMI bound}
\end{align}

Numerical values of the right side of (\ref{CMI bound}) for different values of $\epsilon$ are given in Table \ref{Table of bounds} (CMI bound).
Note that indeed $I([W]_k;X_{T})\rightarrow I(W;X_T)=\infty$ as $k\to \infty$. However, the slow rate of that convergence and the existence of the $2^{-k}$ term makes the sum not only finite, but very small. In fact, as $\epsilon\to 0$, the right side of (\ref{CMI bound}) tends to $0$ as well.

It is interesting to notice that for this toy example, the exact values of $\ex\left[\sup_{\phi\in[0,2\pi)}X_{\phi}\right]$ and $\ex[X_W]$ can be computed. As $\sup_{\phi\in[0,2\pi)}X_{\phi}$ has a Rayleigh distribution, we have $\ex\left[\sup_{\phi\in[0,2\pi)}X_{\phi}\right]=\sqrt{\frac{\pi}{2}}=1.253...~ $. Since the noise $Z$ is independent from $X_T$, the effect of its continuous part cancels out, and we have $\ex[X_W]=\epsilon \sqrt{\frac{\pi}{2}}$. See Table \ref{Table of bounds}. 

\begin{table}[H]
  \caption{$\ex [X_W]$ and its upper bounds}
  \label{Table of bounds}
  \centering
  \begin{tabular}{c c c c c c c c}
  \toprule
    $\epsilon$ &  $\frac{1}{20}$ & $\frac{1}{30}$ & $\frac{1}{40}$ & $\frac{1}{50}$ & $\frac{1}{100}$ & $\frac{1}{200}$ & $\frac{1}{400}$  \\
    \midrule
    $2\sqrt{I(W;X_T)}$ & $\infty$ & $\infty$ & $\infty$ & $\infty$ & $\infty$ & $\infty$ & $\infty$\\
    \midrule
    Chaining bound & 19.0352 & 19.0352 & 19.0352 & 19.0352 & 19.0352 & 19.0352 & 19.0352\\
    \midrule
   CMI bound   & 1.1013  & 0.7507 & 0.5709 & 0.4612 & 0.2364 & 0.1204 & 0.0610  \\
   \midrule
   $\ex[X_W]$ & 0.0626 & 0.0417 & 0.0313 & 0.0250 & 0.0125 & 0.0062 & 0.0031 \\
    \bottomrule
  \end{tabular}
\end{table}
\end{example}

\begin{remark}\label{PJL Remark}\normalfont Notice that in Example \ref{Running Example} there exists an independent additive noise term $Z$ which has a degenerate part, causing the mutual information bound to blow up. Similarly, as discussed in \cite{PJL}, the mutual information bound for perturbed SGD or any iterative algorithm which adds degenerate noise in each iteration blows up. Example \ref{Running Example} illustrates that combining the mutual information method with the chaining method as in our bound could give tight generalization bounds for such algorithms as well.

\end{remark}
\begin{remark}\normalfont
	It is clear that having degenerate noise is not necessary to observe that the chained mutual information bound is tighter than the mutual information bound; this is just an extreme case for which the mutual information bound blows up. For instance, in Example \ref{Running Example}, one can replace $Z$ with a sequence of continuous random variables which converge to $Z$ in distribution.  
\end{remark}
\section{Proof outline}\label{Proof outline}
Here we provide an  outline of  the proof of Theorem \ref{Chaining MI Random Process}. As noted in Section \ref{Main results}, Theorem \ref{Generalization with chaining MI} follows from Theorem \ref{Chaining MI Random Process}. 

For an arbitrary $k\geq k_1(T)$, consider $\pa_k=\{A_1,A_2,...,A_m\}$. Since $\pa_k$ is a $2^{-k}$-partition of $(T,d)$, by definition there exists a set (or a multiset) $\mathcal{N}_k\triangleq\{a_1,a_2,...,a_m\}\subseteq T$ and a mapping $\pi_{\mathcal{N}_k}:T\to \mathcal{N}_k$ such that $\pi_{\mathcal{N}_k}(t)=a_i$ if $t\in A_i$, and further $d\left(t,\pi_{\mathcal{N}_k}(t)\right)\leq 2^{-k}$, for all $i=1,2,...,m$. Therefore $\mathcal{N}_k$ is a $2^{-k}$-net and $\pi_{\mathcal{N}_k}$ is its associated mapping.  It is also clear that for an arbitrary $t_0\in T$, $\mathcal{N}_{k_0}\triangleq \{t_0\}$ is a $2^{-(k_1(T)-1)}$-net. 
Note that for any integer $n\geq k_1(T)$ we can write
\begin{equation}\label{The chaining sum copy}
X_W=X_{t_0}+\sum_{k=k_1(T)}^n \left(X_{\pi_{\mathcal{N}_k}(W)}-X_{\pi_{\mathcal{N}_{k-1}}(W)}\right)+\left(X_W-X_{\pi_{\mathcal{N}_n}(W)}\right).
\end{equation}
Since by the definition of subgaussian processes the process is centered, we have $\ex [X_{t_0}]=0$. Thus 
\begin{equation}\label{ChainingSum copy}
\ex [X_W]-\ex \left[X_W-X_{\pi_{\mathcal{N}_n}(W)}\right]=\sum_{k=k_1(T)}^n \ex \left[X_{\pi_{\mathcal{N}_k}(W)}-X_{\pi_{\mathcal{N}_{k-1}}(W)}\right].
\end{equation}
For every $k\geq k_1(T)$, $\{X_{\pi_{\mathcal{N}_k}(t)}-X_{\pi_{\mathcal{N}_{k-1}}(t)}\}_{t\in T}$ is a subgaussian process with at most $|\mathcal{N}_k||\mathcal{N}_{k-1}|$ distinct terms, hence a finite process. Based on the triangle inequality,
\begin{align}
d\left(\pi_{\mathcal{N}_k}(t),\pi_{\mathcal{N}_{k-1}}(t)\right)&\leq d\left(t,\pi_{\mathcal{N}_k}(t)\right)+d\left(t,\pi_{\mathcal{N}_{k-1}}(t)\right)\nonumber\\
&\leq 3\times 2^{-k}.
\end{align} 
Note that knowing the value of $\left(\pi_{\mathcal{N}_k}(W),\pi_{\mathcal{N}_{k-1}}(W)\right)$ is enough to determine which one of the random variables $\left\{X_{\pi_{\mathcal{N}_k}(t)}-X_{\pi_{\mathcal{N}_{k-1}}(t)}\right\}_{t\in T}$ is chosen according to $W$. Therefore $\left(\pi_{\mathcal{N}_k}(W),\pi_{\mathcal{N}_{k-1}}(W)\right)$ is playing the role of the random index, and since $X_{\pi_{\mathcal{N}_k}(t)}-X_{\pi_{\mathcal{N}_{k-1}}(t)}$ is $d^2\left(\pi_{\mathcal{N}_k}(t),\pi_{\mathcal{N}_{k-1}}(t)\right)$-subgaussian, based on Theorem \ref{Xu Raginsky Theorem}, an application of data processing inequality and by summation, we obtain
\begin{equation}
\sum_{k=k_1(T)}^n\ex \left[X_{\pi_{\mathcal{N}_k}(W)}-X_{\pi_{\mathcal{N}_{k-1}}(W)}\right]\leq \sum_{k=k_1(T)}^n 3\sqrt{2}\times 2^{-k}\sqrt{I(\pi_{\mathcal{N}_k}(W),\pi_{\mathcal{N}_{k-1}}(W);X_T)}.\label{chaining mutual information ineq}
\end{equation}

Notice the chain of mutual information terms in the right side of (\ref{chaining mutual information ineq}). Since $\{\pa_k\}_{k=k_1(T)}^{\infty}$ is an increasing sequence of partitions, for any $t\in T$, knowing $\mathcal{N}_k(t)$ will uniquely determine $\mathcal{N}_{k-1}(t)$. Therefore
\begin{align}
I\left(\pi_{\mathcal{N}_k}(W),\pi_{\mathcal{N}_{k-1}}(W);X_{T}\right)&=I\left(\pi_{\mathcal{N}_k}(W);X_{T}\right)\\
							                                                      		    &= I\left([W]_k;X_{T}\right).
\end{align}
The rest of the proof follows from the definition of separable processes (Definition \ref{Separable process}). For more details, see proof of Theorem \ref{ChainingIneqMITheorem} in Section \ref{Proof of Chaining MI Theorem} of the supplementary material.  

\section{Additional result: small subset property}\label{Small set section}
We adjusted the conservative chaining method in random processes theory to learning problems by taking into account information about the algorithm, with the chained mutual information method.
In this section, we state a result in which such information could make the bounds much tighter.

It is known that for linear models, the stochastic gradient descent (SGD) algorithm always converges to a solution with small norm \cite{Zhang}. Inspired by this observation, we tighten Dudley's inequality (Theorem \ref{DudleyTheorem}), given the following regularization property: the output $W$ of an algorithm, with high probability, chooses a hypothesis from a subset of the hypothesis set with small covering numbers: 
\begin{theorem}[Small subset property]\label{PartitionTheorem}
Assume that $\{X_t\}_{t\in T}$ is a separable subgaussian process on the bounded metric space $(T,d)$. Let $\{T_1, T_2\}$ be a partition of $T$ and assume that $W$ is a random variable taking values on $T$ with $\pr[W\in T_1]=\alpha$. Then we have
\begin{align}\label{Partition Theorem inequality}
\ex[X_W]\leq &6\sum_{k=k_1(T)}^{\infty}2^{-k}\sqrt{\alpha\log N( T_1,d,2^{-k})+(1-\alpha)\log N( T_2,d,2^{-k})+H(\alpha)}.
\end{align}
\end{theorem}
Proof of Theorem \ref{PartitionTheorem} appears in Section \ref{Proof of Chaining MI Theorem} of the supplementary material. Note that the right side of (\ref{Partition Theorem inequality}) becomes much smaller than Dudley's bound when $\alpha$ is close to $1$ and the covering numbers of $T_1$ (the small subset) are much smaller than the covering numbers of $T_2$. 
\begin{remark}\normalfont
One can upper bound the right side of (\ref{Partition Theorem inequality}) by replacing $N(T_2,d,2^{-k})$ with $N(T,d,2^{-k})$. This is particularly useful when bounding the latter is easier than the former.
\end{remark} 
\section{Conclusion}
We combined ideas from information theory and from high dimensional probability to obtain a generalization bound that takes into account both the dependencies between the hypotheses and the dependence between the input and the output of a learning algorithm. We showed on an example that our chained mutual information bound significantly outperforms previous bounds and gets close to the true generalization error. Under a natural regularization property of the learning algorithm, we provided a corollary of our bound which tightens Dudley's inequality; i.e. when the learning algorithm chooses its output from a small subset of hypotheses with high probability.
\section{Acknowledgments}
We gratefully acknowledge discussions with Ramon van Handel on the topic of chaining. This work was partly supported by the NSF CAREER Award CCF-1552131.
\bibliographystyle{unsrt}
\bibliography{biblio}

\begin{thebibliography}{10}

\bibitem{Zhang}
C.~Zhang{,} S.~Bengio{,} M.~Hardt{,}~B. Recht and O.~Vinyals.
\newblock Understanding deep learning requires rethinking generalization.
\newblock In {\em International Conference on Learning Representations (ICLR)},
  Apr. 2017.

\bibitem{belkin2018understand}
M.~Belkin, S.~Ma, and S.~Mandal.
\newblock To understand deep learning we need to understand kernel learning.
\newblock {\em arXiv preprint arXiv:1802.01396}, 2018.

\bibitem{bousquet2004introduction}
O.~Bousquet, S.~Boucheron, and G.~Lugosi.
\newblock Introduction to statistical learning theory.
\newblock In {\em Advanced Lectures on Machine Learning}, pages 169--207.
  Springer, 2004.

\bibitem{SSS}
S.~Shalev-Shwartz and S.~Ben-David.
\newblock {\em Understanding Machine Learning: From Theory to Algorithms}.
\newblock Cambridge University Press, 2014.

\bibitem{McAllester}
D.~A. McAllester.
\newblock Some {PAC-Bayesian} theorems.
\newblock {\em Machine Learning}, 37(3):355--363, 1999.

\bibitem{bousquet2002stability}
O.~Bousquet and A.~Elisseeff.
\newblock Stability and generalization.
\newblock {\em Journal of Machine Learning Research}, 2(Mar):499--526, 2002.

\bibitem{compression}
N.~Littlestone and M.~Warmuth.
\newblock Relating data compression and learnability.
\newblock Technical report, University of California, Santa Cruz, 1986.

\bibitem{Russo}
D.~Russo and J.~Zou.
\newblock How much does your data exploration overfit? controlling bias via
  information usage.
\newblock {\em arXiv preprint arXiv:1511.05219}, 2015.

\bibitem{Ramon}
R.~van Handel.
\newblock Probability in high dimension.
\newblock {\em [Online]. Available:
  \url{https://www.princeton.edu/~rvan/APC550.pdf}}, Dec.~21 2016.

\bibitem{Vershynin}
R.~Vershynin.
\newblock {\em High-Dimensional Probability: An Introduction with Applications
  in Data Science}.
\newblock Cambridge Series in Statistical and Probabilistic Mathematics.
  Cambridge University Press, 2018.

\bibitem{TalagrandBook}
M.~Talagrand.
\newblock {\em Upper and Lower Bounds for Stochastic Processes: Modern Methods
  and Classical Problems}, volume~60.
\newblock Springer Science \& Business Media, 2014.

\bibitem{Lugosi}
S.~Boucheron{,}~G. Lugosi{,} and P.~Massart.
\newblock {\em Concentration Inequalities: A Nonasymptotic Theory of
  Independence}.
\newblock Oxford University Press, 2013.

\bibitem{Dudley}
R.~M. Dudley.
\newblock The sizes of compact subsets of {Hilbert} space and continuity of
  {Gaussian} processes.
\newblock {\em Journal of Functional Analysis}, 1(3):290--330, 1967.

\bibitem{Kawaguchi}
K.~Kawaguchi{,} L.~P. Kaelbling and Y.~Bengio.
\newblock Generalization in deep learning.
\newblock {\em arXiv preprint arXiv:1710.05468}, 2017.

\bibitem{bartlett2005local}
{P. L. Bartlett, O. Bousquet, and S. Mendelson}.
\newblock Local {Rademacher} complexities.
\newblock {\em The Annals of Statistics}, 33(4):1497--1537, 2005.

\bibitem{Jiao}
J.~Jiao{,}~Y. Han and T.~Weissman.
\newblock Dependence measures bounding the exploration bias for general
  measurements.
\newblock In {\em Proc. of IEEE Symposium on Information Theory (ISIT)}, pages
  1475--1479, Aachen, Germany, June 2017.

\bibitem{Jiao2}
J.~Jiao{,}~Y. Han and T.~Weissman.
\newblock Generalizations of maximal inequalities to arbitrary selection rules.
\newblock {\em arXiv preprint arXiv:1708.09041}, 2017.

\bibitem{Raginsky}
A.~Xu and M.~Raginsky.
\newblock Information-theoretic analysis of generalization capability of
  learning algorithms.
\newblock In {\em Advances in Neural Information Processing Systems (NIPS)},
  pages 2524--2533, Dec. 2017.

\bibitem{PJL}
A.~Pensia{,}~V. Jog and P.~Loh.
\newblock Generalization error bounds for noisy, iterative algorithms.
\newblock {\em arXiv preprint arXiv:1801.04295}, 12~Jan 2018.

\bibitem{Moran}
R.~Bassily{,} S.~Moran{,} I.~Nachum{,}~J. Shafer and A.~Yehudayoff.
\newblock Learners that leak little information.
\newblock {\em arXiv preprint arXiv:1710.05233}, 2017.

\bibitem{audibert2004pac}
J.~Audibert and O.~Bousquet.
\newblock {PAC-Bayesian} generic chaining.
\newblock In {\em Advances in Neural Information Processing Systems (NIPS)},
  pages 1125--1132, 2004.

\bibitem{audibert2007combining}
J.~Audibert and O.~Bousquet.
\newblock Combining {PAC-Bayesian} and generic chaining bounds.
\newblock {\em Journal of Machine Learning Research}, 8(Apr):863--889, 2007.

\bibitem{Fernique}
X.~Fernique.
\newblock Evaluations de processus {Gaussiens} composes.
\newblock In {\em Probability in Banach Spaces}, pages 67--83. Springer, 1976.

\bibitem{Cover}
T.~M. Cover and J.~A. Thomas.
\newblock {\em Elements of Information Theory}.
\newblock John Wiley \& Sons, 2012.

\end{thebibliography}

\newpage
\begin{appendix}

In section \ref{Finite Processes} which deals with finite random processes and which serves as the basic foundation of chaining, the known results of maximal inequality (Proposition \ref{MaximalInequalityTheorem}) and its improvement via mutual information (Theorem \ref{FiniteExpectationMI}) are reviewed. Then we give a condition for a random process in Corollary \ref{AbsoluteValueFiniteCorollary}, for which the result of Theorem \ref{FiniteExpectationMI} can be improved by upper \emph{and} lower bounding $\ex[X_W]$. The aforementioned results concern $\ex[X_W]$; in Theorem \ref{FiniteTailMI} we obtain inequalities for the tail behavior of $X_W$.  

In the next step of building upon the results of section \ref{Finite Processes}, to be able to handle infinite processes, in section \ref{LipschitzProcSection} we introduce the notion of \emph{$\epsilon$-nets} (see Definition \ref{epsilonnet}) and its related definitions, and in Theorem \ref{LipschitzWithMITheorem} we upper bound $\ex[X_W]$ for \emph{Lipschitz processes} (see Definition \ref{lipschitzprocess}) using mutual information. This is the strengthened version of the so-called $\epsilon$-net argument, with the usage of mutual information. Remark \ref{AbsoluteValueLipschitzRemark} discusses upper bounding $|\ex[X_W]|$ for Lipschitz processes.

In the last step, in section \ref{Proof of Chaining MI Theorem}, we loosen the ``almost sure'' Lipschitz condition of the dependencies of the random variables of a process to a ``in probability'' condition, defined as subgaussian processes (see Definition \ref{SubgaussianProcessDefinition2}). After reviewing the classical chaining result of Dudley's inequality (Theorem \ref{DudleyTheorem2}), we combine the mutual information method and the chaining method in Theorem \ref{ChainingIneqMITheorem} for subgaussian processes, and in Thoerem \ref{Chaining Ineq Psi Theorem} for more general processes.
\section{Preliminaries}
\begin{definition}[Cumulant generating function] \label{CumulantGenFunctionDefinition} Let $X$ be a real valued random variable. The \emph{cumulant generating function} of $X$ is defined as $\Lambda_X(\lambda)\triangleq\log \ex[e^{\lambda X}]$ for all $\lambda\in\mathbb{R}$.
\end{definition}
The following lemma is a well known fact about the cumulant generating function:
\begin{lemma}\label{CumulantGFExpectation}
Let $X$ be a random variable. Then its cumulant generating function $\Lambda_X$ is convex, $\Lambda_X(0)=0$ and $\Lambda'_X(0)=\ex [X]$.
\end{lemma}
An important and widely used class of random variables is the class of subgaussian random variables:
\begin{definition}[Subgaussian random variables]\label{Subgaussian random variables} The random variable $X$ is called \emph{$\sigma^2$-subgaussian} if $\ex[e^{\lambda(X-\ex X)}]\leq e^{\frac{\lambda^2\sigma^2}{2}}$ for all $\lambda\in\mathbb{R}$. In particular, if $X$ is $\sigma^2$-subgaussian and $\ex[X]=0$, then its cumulant generating function satisfies $\Lambda_X(\lambda)\leq \frac{\lambda^2\sigma^2}{2}$ for all $\lambda\in\mathbb{R}$. The constant $\sigma^2$ is called the \emph{variance proxy}.
\end{definition}

We will use the notion of Legendre dual, defined as follows, in our bounds.
\begin{definition}[Legendre dual] \label{LegendreDualDefinition}
For a convex function $\psi: \mathbb{R}_{+}\to \mathbb{R}$, the \emph{Legendre dual} $\psi^*:\mathbb{R}\rightarrow\mathbb{R}$ is defined as 
\begin{equation}
\psi^*(x)\triangleq\sup_{\lambda\geq 0}\{\lambda x-\psi(\lambda)\} \textrm{ ~for all ~} x\in\mathbb{R}.
\end{equation}
\end{definition}
For a proof of the next lemma see \cite[p. 115]{Ramon}:
\begin{lemma}[Legendre dual properties]\label{LegendreDualProperties} Let $\psi: \mathbb{R}_{+}\rightarrow \mathbb{R}$ be a convex function and $\psi(0)=\psi'(0)=0$. Then $\psi^*(x)$ is a convex, strictly increasing, nonnegative and unbounded function for $x\geq 0$, and $\psi^*(0)=0$. Therefore its inverse $\ps(y)$ is well defined for $y\geq 0$. 
\end{lemma}
From Definition \ref{Subgaussian random variables}, if $X$ is $\sigma^2$-subgaussian and $\ex[X]=0$ then $\Lt\leq \frac{\lambda^2\sigma^2}{2}$. The following lemma gives the Legendre dual inverse of $\psi(\lambda)\triangleq\frac{\lambda^2\sigma^2}{2}$. 
\begin{lemma} Let $\psi(\lambda)\triangleq\frac{\lambda^2\sigma^2}{2}$ for all $\lambda\geq 0$. Then $\ps(x)=\sqrt{2\sigma^2 x}$ for all $x\in\mathbb{R}$.
\end{lemma}
The following is the well-known Chernoff bound:
\begin{lemma}[Chernoff]\label{ChernoffBoundLemma} Let $X$ be a random variable, and $\psi$ be a function such that $\Lambda_{X}(\lambda)\leq \psi(\lambda)$ for all $\lambda\geq 0$. Then
\begin{equation}
\pr[X\geq x]\leq e^{-\psi^*(x)} \textrm{ ~ for all ~ } x\in \mathbb{R}.
\end{equation}
\end{lemma}
The variational representation of relative entropy is a useful information theoretic tool:
\begin{theorem}[Variational representation of relative entropy]\label{Donsker-Varadhan} Let $X$ and $Y$ be random variables taking values on $\mathcal{A}$ with distributions $P_X$ and $P_Y$, respectively. Then
\begin{equation}
D(P_X\|P_Y)=\max_{f\in\mathtt{F}}\left\{\ex\left[f(X)\right]-\log\ex\left[e^{f(Y)}\right]\right\},
\end{equation}
where the maximum is with respect to $\mathtt{F}=\left\{f:\mathcal{A}\rightarrow \mathbb{R} \mathrm{ ~s. t.~ } \ex[e^{f(Y)}]<\infty\right\}$, and is achieved by $f^*(a)=\imath_{X\|Y}(a)$.
\end{theorem}
\section{Finite processes (random vectors)}\label{Finite Processes}
In this section we consider a random process $\{X_t\}_{t\in T}$ where $T$ is a finite set. The following is a well known result (see \cite[Theorem 2.5]{Lugosi}):
\begin{proposition}[Maximal inequality]\label{MaximalInequalityTheorem}
Let $\{X_t\}_{t\in T}$ be a random process and $T$ a finite set. Assume that $\Lambda_{X_t}(\lambda)\leq \psi(\lambda)$ for all $\lambda\geq 0$ and $t\in T$, where $\psi$ is convex and $\psi(0)=\psi'(0)=0$. Then 
\begin{equation}\label{FiniteExpectationCardinalityInequality}
\ex\left[\sup_{t\in T}X_t\right]\leq {\psi^*}^{-1}(\log |T|).
\end{equation} 
In particular, if $X_t$ is $\sigma^2$-subgaussian and $\ex[X_t]=0$ for every $t\in T$, then 
\begin{equation}
\ex\left[\sup_{t\in T}X_t\right]\leq \sqrt{2\sigma^2\log |T|}.
\end{equation} 
\end{proposition}

\begin{remark}\normalfont \label{RemarkMeanIsZero}
Note that based on Lemma \ref{CumulantGFExpectation}, for all $t\in T$, the condition $\Lt\leq \psi(\lambda)$ for all $\lambda\geq 0$ and $\psi'(0)=0$ implies that $\ex[X_t]=0$.
\end{remark}
\begin{proposition}
If in addition to the assumptions of Proposition \ref{MaximalInequalityTheorem}, we assume that $\Lambda_{X_t}(-\lambda)\leq \psi(\lambda)$ for all $\lambda\geq 0$ and $t\in T$, then we have 
\begin{equation}
\ex\left[\sup_{t\in T}|X_t|\right]\leq {\psi^*}^{-1}\left(\log (2|T|)\right).
\end{equation}
In particular, if  $X_t$ is $\sigma^2$-subgaussian and $\ex[X_t]=0$ for every $t\in T$, then 
\begin{equation}
\ex\left[\sup_{t\in T}|X_t|\right]\leq \sqrt{2\sigma^2\log \left(2|T|\right)}.
\end{equation} 
\end{proposition}
\begin{proof}
Apply Proposition \ref{MaximalInequalityTheorem} on the random process $\{X_t\}_{t\in T}\cup \{-X_t\}_{t\in T}$.
\end{proof}
The next result bounds $\ex[X_W]$, where $W$ is a random variable taking values on $T$:
\begin{theorem}\cite{Russo}, \cite{Jiao}\label{FiniteExpectationMI}
Let $\{X_t\}_{t\in T}$ be a random process and $T$ a finite set. Assume that $\Lt\leq \psi(\lambda)$ for all $\lambda\geq 0$ and $t\in T$, where $\psi$ is convex and $\psi(0)=\psi'(0)=0$, and let $W$ be a random variable taking values on $T$. Then 
\begin{equation}
\ex[X_W]\leq {\psi^*}^{-1}(I(W;X_{T})). \label{FiniteExpectationMIInequality}
\end{equation} 
In particular, if $X_t$ is $\sigma^2$-subgaussian and $\ex[X_t]=0$ for every $t\in T$, then
\begin{equation}\label{RussoInequality}
\ex[X_W]\leq \sqrt{2\sigma^2 I(W;X_{T})}.
\end{equation} 
\end{theorem}
Based on Lemma \ref{LegendreDualProperties}, $\ps$ is an increasing function. Therefore one can replace $I(W;X_{T})$ with any larger quantity in the right side of (\ref{FiniteExpectationMIInequality}). For example,
\begin{align}
\ex[X_W]&\leq {\psi^*}^{-1}(I(W;X_{T}))\nonumber\\
              &\leq \ps(H(W)).
\end{align}
Since $W$ takes values on $T$, we have $H(W)\leq \log|T|$. Therefore the right side of (\ref{FiniteExpectationMIInequality}) is not larger than the right side of (\ref{FiniteExpectationCardinalityInequality}).

Based on Lemma \ref{LegendreDualProperties}, the right side of (\ref{FiniteExpectationMIInequality}) is zero if and only if $I(W;X_{T})=0$, i.e. $W$ is independent of $X_{T}$. In this case, (\ref{FiniteExpectationMIInequality}) turns into an equality: based on Remark \ref{RemarkMeanIsZero} we have $\ex[X_t]=0$ for all $t\in T$, hence $\ex[X_W]=\ex[\ex[X_W|W]]=0$.

Now, by adding an assumption, we prove upper \emph{and} lower bounds for $\ex[X_W]$, and an upper bound for $\ex[|X_W|]$. We should mention that the proof of part (b) of the following proposition is similar to the proof of Theorem 4 in \cite{Raginsky}.
\begin{proposition}\label{AbsoluteValueFiniteCorollary} If in addition to the assumptions of Theorem \ref{FiniteExpectationMI}, we assume that $\Lambda_{X_t}(-\lambda)\leq \psi(\lambda)$ for all $\lambda\geq 0$ and $t\in T$, then we have
\begin{enumerate}[(a)]
\item \begin{equation}
|\ex[X_W]|\leq {\psi^*}^{-1}(I(W;X_{T})),
\end{equation}
\item \begin{equation}
\ex[|X_W|]\leq \ps\left(I(W;X_{T})+\log 2\right).
\end{equation}
\end{enumerate}
\end{proposition}
\begin{proof}\leavevmode
\begin{enumerate}[(a)]
\item Apply Theorem \ref{FiniteExpectationMI} to the process $\{-X_t\}_{t\in T}$, while noting that $\Lambda_{-X_t}(\lambda)=\Lambda_{X_t}(-\lambda)$ for all $\lambda\geq 0$ and $t\in T$, and $I(W;-X_{T})=I(W;X_{T})$, since mutual information is invariant to one-to-one functions.

\item Define the random process $\overline{X}=\{X_{t,w}\}_{\substack{t\in T\\ w\in\{0,1\}}}$ such that
\[ X_{t,w}\triangleq\begin{cases} 
      X_t &  t\in T, w=0 \\
      -X_t & t\in T, w=1 
   \end{cases}
\] 
and let $R$ be a random variable taking values on $\{0,1\}$ such that 
\[ R=\begin{cases} 
      0 &  \mathrm{if~} X_W\geq 0 \\
      1 &  \mathrm{if~} X_W<0 
   \end{cases}
\cdot\] 
Based on Theorem \ref{FiniteExpectationMI} applied on the random process $\overline{X}$ and random variables $W$ and $R$, and based on the chain rule of entropy, we get
\begin{align}
\ex[|X_W|]&=\ex[X_{W,R}]\\
			   &\leq \ps(I(W,R ;\overline{X}))\\
			   &=\ps(H(W,R)-H(W,R|\overline{X}))\\
			   &=\ps((H(W)+H(R|W))-(H(W|\overline{X})+H(R|W,\overline{X})))\\
			   &=\ps((H(W)+H(R|W))-H(W|\overline{X}))\\
			   &=\ps(H(W)-H(W|X_{T})+H(R|W))\\
			   &=\ps(I(W;X_{ T})+H(R|W))\\
			   &\leq\ps(I(W;X_{ T})+H(R))\\
			   &\leq \ps(I(W;X_{ T})+\log 2).
\end{align}
\end{enumerate}

\end{proof}
\begin{corollary}\normalfont\label{ExampleAbsoluteValueRemark}
If $ T$ is a finite set, $\psi(\lambda)\triangleq \frac{\lambda^2\sigma^2}{2}$, and for all $t\in T$, $X_t$ is $\sigma^2$-subgaussian and $\ex[X_t]=0$, then the conditions of Theorem \ref{AbsoluteValueFiniteCorollary} is satisfied, and (\ref{RussoInequality}) can be improved to 
\begin{equation}
|\ex[X_W]|\leq \sqrt{2\sigma^2 I(W;X_{ T})},
\end{equation}
as was shown in \cite{Russo}.
\end{corollary}
The previous results concerned $\ex[\sup_{t\in T} X_t]$ and $\ex[X_W]$. We now state a result for estimating the tail probability of $\sup_{t\in T}X_t$:
\begin{proposition}\cite{Ramon}\label{FiniteTail}
Let $\{X_t\}_{t\in T}$ be a random process and $ T$ a finite set. Assume that $\Lt\leq \psi(\lambda)$ for all $\lambda\geq 0$ and $t\in  T$, where $\psi$ is convex and $\psi(0)=\psi'(0)=0$. Then
\begin{equation}
\pr\left[\sup_{t\in T}X_t\geq {\psi^*}^{-1}(\log| T|+u)\right]\leq e^{-u}  \textrm{   for all   } u\geq 0.
\end{equation}
In particular, if $X_t$ is $\sigma^2$-subgaussian and $\ex[X_t]=0$ for every $t\in  T$, then
\begin{equation}
\pr\left[\sup_{t\in T}X_t\geq \sqrt{2\sigma^2\log| T|}+x\right]\leq e^{-\frac{x^2}{2\sigma^2}}  \textrm{   for all   } x\geq 0.
\end{equation}
\end{proposition}

We estimate the tail probability of $X_W$ in the following theorem:
\begin{theorem}\label{FiniteTailMI}
Let $\{X_t\}_{t\in T}$ be a random process and $ T$ a finite set. Assume that $\Lt\leq \psi(\lambda)$ for all $\lambda\geq 0$ and $t\in  T$, where $\psi$ is convex and $\psi(0)=\psi'(0)=0$, and let $W$ be a random variable taking values on $ T$. Then for all $u\geq 0$,
\begin{align}
&\pr\left[X_W\geq {\psi^*}^{-1}(I(W;X_{ T})+u)\right]\nonumber\\
&\leq \min \left\{\frac{I(W;X_{ T})+\log \left(2-e^{-I(W;X_{ T})-u}\right)}{I(W;X_{ T})+u}, e^{\log | T|-I(W;X_{ T})-u}\right\}.\label{GeneralCaseTailInequality}
\end{align}
In particular, if $X_t$ is $\sigma^2$-subgaussian and $\ex[X_t]=0$ for every $t\in  T$, then for all $x\geq 0$,
\begin{align}
&\pr\left[X_W\geq \sqrt{2\sigma^2I(W;X_{ T})}+x\right]\nonumber\\
&\leq \min \left\{\frac{I(W;X_{ T})+\log \left(2-e^{-I(W;X_{ T})-\frac{x^2}{2\sigma^2}}\right)}{I(W;X_{ T})+\frac{x^2}{2\sigma^2}}, e^{\log | T|-I(W;X_{ T})-\frac{x^2}{2\sigma^2}}\right\}.\label{SubgaussianCaseTailInequality}
\end{align}
\end{theorem}
\begin{proof}
Analogous to the proof of Theorem \ref{FiniteExpectationMI} in \cite{Russo}, \cite{Jiao}, we invoke the variational representation of relative entropy (Theorem \ref{Donsker-Varadhan}) in our proof.

Define $n\triangleq | T|$ and without loss of generality, let $ T\triangleq\{1,2,...,n\}$. Note that
\begin{align}
&\pr\left[X_W\geq {\psi^*}^{-1}(I(W;X_{ T})+u)\right]\nonumber\\
&=\sum_{i=1}^n\pr\left[X_W\geq {\psi^*}^{-1}(I(W;X_{ T})+u)\middle|W=i\right]\pr\left[W=i\right]\label{LawTotalProb2}\\
&=\sum_{i=1}^n\pr\left[X_i\geq {\psi^*}^{-1}(I(W;X_{ T})+u)\middle|W=i\right]\pr[W=i].\label{LawTotalProb}
\end{align}  
Define
\begin{equation}
f(a)\triangleq \zeta\mathbbm{1}_{\left\{a\geq \ps(I(W;X_{ T})+u)\right\}},
\end{equation}
where $\zeta>0$ is an arbitrary real number. Choose an arbitrary $1\leq i\leq n$, and define random variable $X$ such that $P_X=P_{X_i|W=i}$. We have
\begin{align}
&\zeta\pr\left[X_i\geq {\psi^*}^{-1}(I(W;X_{ T})+u)\middle|W=i\right]\nonumber\\
&\hspace{20mm}=\ex[f(X)]\\
&\hspace{20mm}\leq D(P_X\|P_{X_i})+\log\ex[e^{f(X_i)}]\label{DVApplication}\\
&\hspace{20mm}=D(P_{X_i|W=i}\|P_{X_i})+\log \ex[e^{f(X_i)}]\\
&\hspace{20mm}=D(P_{X_i|W=i}\|P_{X_i})+\log\left(e^{\zeta}\pr\left[X_i\geq \ps(I(W;X_{ T})+u)\right]\right.\nonumber\\
&\hspace{20mm}\left.\quad+\pr\left[X_i<\ps(I(W;X_{ T})+u)\right]\right)\\
&\hspace{20mm}=D(P_{X_i|W=i}\|P_{X_i})\nonumber\\
&\hspace{20mm}\quad+\log\left((e^{\zeta}-1)\pr\left[X_i\geq \ps(I(W;X_{ T})+u)\right]+1\right)\\
&\hspace{20mm}\leq D(P_{X_i|W=i}\|P_{X_i})+\log\left((e^{\zeta}-1)e^{-I(W;X_{ T})-u}+1\right)\label{DerivedFromChernoffBound1}\\
&\hspace{20mm}\leq D(P_{X_{ T}|W=i}\|P_{X_{ T}})+\log\left((e^{\zeta}-1)e^{-I(W;X_{ T})-u}+1\right),\label{DataProcessingIn}
\end{align}
where (\ref{DVApplication}) is based on Theorem \ref{Donsker-Varadhan}, (\ref{DerivedFromChernoffBound1}) is based on Lemma \ref{ChernoffBoundLemma}  and (\ref{DataProcessingIn}) is based on the data processing inequality for relative entropy.
Therefore
\begin{align}
&\pr\left[X_i\geq {\psi^*}^{-1}(I(W;X_{ T})+u)\middle|W=i\right]\nonumber\\
&\leq\frac{1}{\zeta}\left(D(P_{X_{ T}|W=i}\|P_{X_{ T}})+\log\left((e^{\zeta}-1)e^{-I(W;X_{ T})-u}+1\right)\right). \label{IneqA}
\end{align}
Since $i$ was chosen arbitrarily, (\ref{IneqA}) holds for all $i=1,2,...,n$. Thus, based on (\ref{LawTotalProb2}) and (\ref{LawTotalProb}) we have
\begin{align}
\pr\left[X_W\geq {\psi^*}^{-1}(I(W;X_{ T})+u)\right]&\leq \frac{1}{\zeta}\left(\sum_{i=1}^n D(P_{X_{ T}|W=i}\|P_{X_{ T}})\pr[W=i]\right.\nonumber\\
&\quad\quad+\left.\log \left((e^{\zeta}-1)e^{-I(W;X_{ T})-u}+1\right)\right)\\
&\quad=\frac{1}{\zeta}\left(I(W;X_{ T})+\log \left((e^{\zeta}-1)e^{-I(W;X_{ T})-u}+1\right)\right).\label{IneqB}
\end{align}
Since (\ref{IneqB}) holds for arbitrary $\zeta>0$, we can infimize the right side of (\ref{IneqB}) over $\zeta$ to obtain
\begin{align}
&\pr\left[X_W\geq {\psi^*}^{-1}(I(W;X_{ T})+u)\right]\nonumber\\
&\hspace{20mm}\leq\inf_{\zeta>0}\left\{\frac{1}{\zeta}\left(I(W;X_{ T})+\log \left((e^{\zeta}-1)e^{-I(W;X_{ T})-u}+1\right)\right)\right\}.\label{InfimumExpression}
\end{align}
Now, we upper bound the right side of (\ref{InfimumExpression}) by choosing $\zeta\leftarrow I(W;X_{ T})+u$, to get
\begin{align}
\pr\left[X_W\geq {\psi^*}^{-1}(I(W;X_{ T})+u)\right]\leq\frac{I(W;X_{ T})+\log\left(2-e^{-I(W;X_{ T})-u}\right)}{I(W;X_{ T})+u},\label{FirstMinBound}
\end{align}
which is one of the terms in the right side of (\ref{GeneralCaseTailInequality}). To prove the other upper bound in (\ref{GeneralCaseTailInequality}), note that
\begin{align}
\pr\left[X_W\geq {\psi^*}^{-1}(I(W;X_{ T})+u)\right]&=\sum_{i=1}^n\pr\left[X_W\geq {\psi^*}^{-1}(I(W;X_{ T})+u),W=i\right]\\
&=\sum_{i=1}^n\pr\left[X_i\geq {\psi^*}^{-1}(I(W;X_{ T})+u),W=i\right]\\
&\leq \sum_{i=1}^n\pr\left[X_i\geq {\psi^*}^{-1}(I(W;X_{ T})+u)\right]\\
&\leq ne^{-I(W;X_{ T})-u}\label{DerivedChernoffBound2}\\
&=e^{\log | T|-I(W;X_{ T})-u},\label{SecondMinBound}
\end{align}
where (\ref{DerivedChernoffBound2}) is based on Lemma \ref{ChernoffBoundLemma}. 

For the subgaussian case, note that 
\begin{align}
\ps(\log | T|+u)&=\sqrt{2\sigma^2(\log | T|+u)}\\
					     &\leq \sqrt{2\sigma^2\log | T|}+\sqrt{2\sigma^2u},
\end{align}
therefore, based on (\ref{FirstMinBound}) and (\ref{SecondMinBound}), we get (\ref{SubgaussianCaseTailInequality}).
\end{proof}
Note that our upper bound in (\ref{FirstMinBound}) is slightly stronger than Lemma 4.1 in \cite{Moran}, and our method of proving (\ref{FirstMinBound}) shows that Lemma 4.1 in \cite{Moran} is a corollary of the well known variational representation of relative entropy (Theorem \ref{Donsker-Varadhan}).

\begin{remark}\normalfont
If the assumptions of Proposition \ref{AbsoluteValueFiniteCorollary} hold, then by applying Theorem \ref{FiniteTailMI} on $\{-X_t\}_{t\in T}$, it is straightforward to obtain analogous lower tail bounds for $X_W$. 
\end{remark}
\section{Lipschitz processes and the $\epsilon$-net argument}\label{LipschitzProcSection}
The generalization of the maximal inequality (Proposition \ref{MaximalInequalityTheorem}) to random processes with infinite number of random variables is not useful, since its upper bound blows up. But in many applications, there exists some dependence structure between the random variables of the random process which can be exploited to give better bounds. In this section we define \emph{Lipschitz} structure and mention the $\epsilon$-net argument. Then we show how to tighten that by using mutual information.

\begin{definition}[Lipschitz process] \label{lipschitzprocess} The random process $\{X_t\}_{t\in T}$ is called \emph{Lipschitz} for a metric $d$ on $ T$ if there exists a random variable $C$ such that $|X_t-X_s|\leq Cd(t,s)$ for all $t,s\in  T$.
\end{definition}
Here we give the definitions of $\epsilon$-net and covering number $N( T,d,\epsilon)$:
\begin{definition}[$\epsilon$-net and covering number]\label{epsilonnet} Let $d$ be a metric on the set $T$.
\begin{enumerate}[(a)]
\item A finite set $\mathcal{N}$ is called an \emph{$\epsilon$-net} for $(T,d)$ if there exists a function $\pi_{\mathcal{N}}$ which maps every point $t\in T$ to $\pi_{\mathcal{N}}(t)\in\mathcal{N}$ such that $d(t,\pi_{\mathcal{N}}(t))\leq \epsilon$.
\item The \emph{covering number} for a metric space $(T,d)$ is the smallest cardinality of an $\epsilon$-net for that space, where we denote it by $N(T,d,\epsilon)$. In other words,
\begin{equation}
N(T,d,\epsilon)\triangleq \inf\{|\mathcal{N}|: \mathcal{N} \textrm{ is an } \epsilon\textrm{-net for } (T,d)\}.
\end{equation} 
\item An $\epsilon$-net $\mathcal{N}$ for the metric space $(T,d)$ is called \emph{minimal} if $|\mathcal{N}|=N(T,d,\epsilon)$.
\end{enumerate}
\end{definition}
For Lipschitz processes, the following inequality usually gives better bounds than the maximal inequality (Proposition \ref{MaximalInequalityTheorem}), and it is also referred to as \emph{the $\epsilon$-net argument}: 

\begin{proposition}[Lipschitz maximal inequality]\label{Lipschitz maximal inequality proposition} Assume that $\{X_t\}_{t\in  T}$ is a Lispschitz process for the metric $d$ on $ T$, and $\Lt\leq \psi(\lambda)$ for all $\lambda\geq 0$ and $t\in  T$, where $\psi$ is convex and $\psi(0)=\psi'(0)=0$. Then
\begin{equation}\label{LipschitzWithCNInequality}
\ex\left[\sup_{t\in T}X_t\right]\leq \inf_{\epsilon >0}\left\{\epsilon \ex[C]+ \ps\left(\log N( T, d,\epsilon\right)\right\}.
\end{equation}
\end{proposition}
For a proof of Proposition \ref{Lipschitz maximal inequality proposition} see \cite{Ramon}. The following theorem tightens Proposition \ref{Lipschitz maximal inequality proposition} by using the mutual information method:
\begin{theorem}\label{LipschitzWithMITheorem}
Assume that $\{X_t\}_{t\in T}$ is a Lipschitz process for the metric $d$ on $ T$, and $\Lt\leq \psi(\lambda)$ for all $\lambda\geq 0$ and $t\in  T$, where $\psi$ is convex and $\psi(0)=\psi'(0)=0$. If for all $\epsilon>0$, $\mathcal{N}_{\epsilon}$ is an $\epsilon$-net for $( T,d)$, then
\begin{equation}
\ex [X_W]\leq \inf_{\substack{\epsilon>0\\ \mathcal{N}_{\epsilon}}}\left\lbrace \epsilon\ex [C]+{\psi^*}^{-1}(I(\pi_{\mathcal{N}_{\epsilon}}(W);X_{\mathcal{N}_{\epsilon}}))\right\rbrace, \label{LipschitzWithMIInequality}
\end{equation}
where the infimum is over all $\epsilon>0$ and all $\epsilon$-nets $\mathcal{N}_{\epsilon}$ of $( T,d)$. 
\end{theorem}
\begin{proof}
We have $X_W=(X_W-X_{\pi_{\mathcal{N}_{\epsilon}}(W)})+X_{\pi_{\mathcal{N}_{\epsilon}}(W)}$. Therefore, based on Theorem \ref{FiniteExpectationMI} and Definition \ref{lipschitzprocess}, we have
\begin{align}
\ex[X_W]&=\ex [X_W-X_{\pi_{\mathcal{N}_{\epsilon}}(W)}]+\ex [X_{\pi_{\mathcal{N}_{\epsilon}}(W)}]\\
			 &\leq \ex[|X_W-X_{\pi_{\mathcal{N}_{\epsilon}}(W)}|]+\ex [X_{\pi_{\mathcal{N}_{\epsilon}}(W)}]\\
			 &\leq \epsilon \ex [C]+{\psi^*}^{-1}\left(I(\pi_{\mathcal{N}_{\epsilon}}(W);X_{\mathcal{N}_{\epsilon}})\right)
\end{align}
\end{proof}

\begin{remark}\normalfont
Note that in the infimum in (\ref{LipschitzWithMIInequality}), for all $\epsilon>0$ one can restrict $\mathcal{N}_{\epsilon}$ to be a minimal $\epsilon$-net to conclude that the right side of (\ref{LipschitzWithMIInequality}) is no larger than the right side of (\ref{LipschitzWithCNInequality}), due to Lemma \ref{LegendreDualProperties} and the following inequalities:
\begin{align}
I(\pi_{\mathcal{N}_{\epsilon}}(W);X_{\mathcal{N}_{\epsilon}})&\leq H(\pi_{\mathcal{N}_{\epsilon}}(W))\\
								&\leq \log N( T,d,\epsilon).
\end{align}
\end{remark}

\begin{proposition}\label{Lipschitz with MI Vs Russo}
With the assumptions of Theorem \ref{LipschitzWithMITheorem}, we have
\begin{equation}
\inf_{\substack{\epsilon>0\\ \mathcal{N}_{\epsilon}}}\left\lbrace \epsilon\ex [C]+{\psi^*}^{-1}\left(I(\pi_{\mathcal{N}_{\epsilon}}(W);X_{\mathcal{N}_{\epsilon}})\right)\right\rbrace\leq \ps(I(W;X_{ T})).
\end{equation}
Therefore the bound on $\ex[X_W]$ given in Theorem \ref{LipschitzWithMITheorem} is no larger than the bound given in Theorem \ref{FiniteExpectationMI}.
\end{proposition}
\begin{proof}
For all $\epsilon>0$, based on the chain rule of mutual information (or the data processing inequality), we have
\begin{equation}\label{Lipschitz with MI Vs Russo proof equation 1}
I(\pi_{\mathcal{N}_{\epsilon}}(W);X_{\mathcal{N}_{\epsilon}})\leq I(\pi_{\mathcal{N}_{\epsilon}}(W);X_{ T}).
\end{equation} 
Furthermore, the Markov chain $\pi_{\mathcal{N}_{\epsilon}}(W)\leftrightarrow W \leftrightarrow X_{ T}$ and the data processing inequality for mutual information yield
\begin{equation}\label{Lipschitz with MI Vs Russo proof equation 2}
I(\pi_{\mathcal{N}_{\epsilon}}(W);X_{ T})\leq I(W;X_{ T}).
\end{equation}
Lemma \ref{LegendreDualProperties} along with (\ref{Lipschitz with MI Vs Russo proof equation 1}) and (\ref{Lipschitz with MI Vs Russo proof equation 2}) conclude
\begin{equation}\label{Lipschitz with MI Vs Russo proof equation 3}
\epsilon \ex[C]+\ps(I(\pi_{\mathcal{N}_{\epsilon}}(W);X_{\mathcal{N}_{\epsilon}}))\leq \epsilon \ex[C]+ \ps(I(W;X_{ T})).
\end{equation}
Letting $\epsilon\rightarrow 0$ completes the proof.
\end{proof}

\begin{remark}\normalfont\label{AbsoluteValueLipschitzRemark}
If in addition to the assumptions of Theorem \ref{LipschitzWithMITheorem}, we have $\Lambda_{X_t}(-\lambda)\leq \psi(\lambda)$ for all $\lambda\geq 0$ and $t\in T$ (see Corollary \ref{ExampleAbsoluteValueRemark} for an example), then similar to the proof of Proposition \ref{AbsoluteValueFiniteCorollary}, we can prove
\begin{equation}
\left|\ex [X_W]\right|\leq \epsilon \ex [C]+{\psi^*}^{-1}(I(\pi_{\mathcal{N}}(W);X_{\mathcal{N}})).
\end{equation}
\end{remark}
\section{Chaining mutual information}\label{Proof of Chaining MI Theorem}
We loosen the ``almost sure'' Lipschitz condition of the dependencies of the random variables of a process to a ``in probability'' condition, defined as subgaussian processes:
\begin{definition} [Subgaussian process] \label{SubgaussianProcessDefinition2} The random process $\{X_t\}_{t\in T}$ on the metric space $(T,d)$ is called \emph{subgaussian} if $\ex[X_t]=0$ for all $t\in T$ and 
\begin{equation}
\ex\left[e^{\lambda(X_t-X_s)}\right]\leq e^{\frac12  \lambda^2 d^2(t,s)} \textrm{ ~  for all ~  } t,s\in T,  \lambda\geq 0.
\end{equation} 
\end{definition}
We now state a classical chaining result:
\begin{theorem}[Dudley]\label{DudleyTheorem2} \cite{Dudley}.
Assume that $\{X_t\}_{t\in T}$ is a separable subgaussian process on the bounded metric space $(T,d)$. Then
\begin{equation}\label{DudleyIneq2}
\ex\left[\sup_{t\in T}X_t\right]\leq 6\sum_{k\in \mathbb{Z}}2^{-k}\sqrt{\log N(T,d,2^{-k})}.
\end{equation}
\end{theorem}
By combining the mutual information method and the chaining method, we obtain the following result:
\begin{theorem}\label{ChainingIneqMITheorem}
Assume that $\{X_t\}_{t\in T}$ is a separable subgaussian process on the bounded metric space $(T,d)$ and let $k_0$ be an integer such that $2^{-k_0}\geq \mathrm{diam}(T)$. Let $\{\mathcal{N}_k\}_{k=k_0+1}^{\infty}$ be a sequence of sets, where for each $k>k_0$, $\mathcal{N}_k$ is a $2^{-k}$-net for $(T,d)$. For an arbitrary $t_0\in T$, let $\mathcal{N}_{k_0}\triangleq\{t_0\}$. Assume that $W$ is a random variable which takes values on $T$. We have
\begin{enumerate}[(a)]
\item 
\begin{equation}\label{Chaining core theorem part one}
\ex [X_W] \leq 3\sqrt{2}\sum_{k=k_0+1}^{\infty}2^{-k}\sqrt{I(\pi_{\mathcal{N}_k}(W),\pi_{\mathcal{N}_{k-1}}(W);X_{T})}.
\end{equation}
\item 
\begin{equation}
\ex \left[|X_W-X_{t_0}|\right] \leq 3\sqrt{2}\sum_{k=k_0+1}^{\infty}2^{-k}\sqrt{I(\pi_{\mathcal{N}_k}(W),\pi_{\mathcal{N}_{k-1}}(W);X_{T})+\log 2}.
\end{equation}
\end{enumerate}
\end{theorem}

\begin{proof}\leavevmode
\begin{enumerate}[(a)]
\item
Since $2^{-k_0}\geq \mathrm{diam}(T)$, we have $N(T,d,2^{-k_0})=1$, therefore $\mathcal{N}_{k_0}$ is a $2^{-k_0}$-net for $(T,d)$. Note that for any integer $n>k_0$ we can write
\begin{equation}\label{The chaining sum}
X_W=X_{t_0}+\sum_{k=k_0+1}^n (X_{\pi_{\mathcal{N}_k}(W)}-X_{\pi_{\mathcal{N}_{k-1}}(W)})+(X_W-X_{\pi_{\mathcal{N}_n}(W)}).
\end{equation}
Since by the definition of subgaussian processes the process is centered, we have $\ex [X_{t_0}]=0$. Thus 
\begin{equation}\label{ChainingSum}
\ex [X_W]-\ex [X_W-X_{\pi_{\mathcal{N}_n}(W)}]=\sum_{k=k_0+1}^n \ex [X_{\pi_{\mathcal{N}_k}(W)}-X_{\pi_{\mathcal{N}_{k-1}}(W)}].
\end{equation}
Note that for every $k>k_0$, $\{X_{\pi_{\mathcal{N}_k}(t)}-X_{\pi_{\mathcal{N}_{k-1}}(t)}\}_{t\in T}$ is a subgaussian process with at most $|\mathcal{N}_k||\mathcal{N}_{k-1}|$ distinct terms, hence a finite process. Based on triangle inequality,
\begin{align}
d(\pi_{\mathcal{N}_k}(t),\pi_{\mathcal{N}_{k-1}}(t))&\leq d(t,\pi_{\mathcal{N}_k}(t))+d(t,\pi_{\mathcal{N}_{k-1}}(t))\nonumber\\
&\leq 3\times 2^{-k}.
\end{align} 
Note that knowing the value of $(\pi_{\mathcal{N}_k}(W),\pi_{\mathcal{N}_{k-1}}(W))$ is enough to determine which one of the random variables of $\{X_{\pi_{\mathcal{N}_k}(t)}-X_{\pi_{\mathcal{N}_{k-1}}(t)}\}_{t\in T}$ is chosen according to $W$. Therefore $(\pi_{\mathcal{N}_k}(W),\pi_{\mathcal{N}_{k-1}}(W))$ is playing the role of the random index, and since $X_{\pi_{\mathcal{N}_k}(t)}-X_{\pi_{\mathcal{N}_{k-1}}(t)}$ is $d^2(\pi_{\mathcal{N}_k}(t),\pi_{\mathcal{N}_{k-1}}(t))$-subgaussian, based on Theorem \ref{FiniteExpectationMI}, we have
\begin{align}
&\ex \left[X_{\pi_{\mathcal{N}_k}(W)}-X_{\pi_{\mathcal{N}_{k-1}}(W)}\right]\nonumber\\
&\hspace{5mm}\leq 3\sqrt{2}\times 2^{-k}\left(I(\pi_{\mathcal{N}_k}(W),\pi_{\mathcal{N}_{k-1}}(W);\{X_{\mathcal{N}_k(t)}-X_{\mathcal{N}_{k-1}(t)}\}_{t\in T})\right)^{\frac{1}{2}}.\label{Difference between part a and b}
\end{align}
Based on the chain rule of mutual information, adding random variables to one side of mutual information does not decrease its value. Thus
\begin{equation}\label{Inequality for each k}
\ex [X_{\pi_{\mathcal{N}_k}(W)}-X_{\pi_{\mathcal{N}_{k-1}}(W)}]\leq 3\sqrt{2}\times 2^{-k}\left(I(\pi_{\mathcal{N}_k}(W),\pi_{\mathcal{N}_{k-1}}(W);X_{\mathcal{N}_{k}}-X_{\mathcal{N}_{k-1}})\right)^{\frac{1}{2}}.
\end{equation}
From (\ref{ChainingSum}) and by using (\ref{Inequality for each k}) for each $k=k_0+1,\dots,n$, we conclude
\begin{equation}
\ex [X_W]-\ex [X_W-X_{\pi_{\mathcal{N}_n}(W)}]\leq\sum_{k=k_0+1}^n 3\sqrt{2}\times 2^{-k}\left(I(\pi_{\mathcal{N}_k}(W),\pi_{\mathcal{N}_{k-1}}(W);X_{\mathcal{N}_k}-X_{\mathcal{N}_{k-1}})\right)^{\frac{1}{2}}.\label{ChainingSum2}
\end{equation}
Note that $|\ex [X_W-X_{\pi_{\mathcal{N}_n}(W)}]|\leq \ex [\sup_{t\in T}(X_t-X_{\pi_{\mathcal{N}_n}(t)})]$, and since the process is separable, we have 
\begin{equation}
\lim_{n\rightarrow\infty}\ex [\sup_{t\in T}(X_t-X_{\pi_{\mathcal{N}_n}(t)})]=0,
\end{equation}
(see proof of Theorem 5.24 in \cite{Ramon}.) Hence
\begin{equation}\label{ChainingRemainderTerm}
\lim_{n\rightarrow\infty}\ex [X_W-X_{\pi_{\mathcal{N}_n}(W)}]=0. 
\end{equation}
Based on (\ref{ChainingSum2}) and (\ref{ChainingRemainderTerm}), we get
\begin{align}\label{ChainingIneq}
&\ex[X_W]\leq 3\sqrt{2}\sum_{k=k_0+1}^{\infty}2^{-k}\left(I(\pi_{\mathcal{N}_k}(W),\pi_{\mathcal{N}_{k-1}}(W);X_{\mathcal{N}_k}-X_{\mathcal{N}_{k-1}})\right)^{\frac{1}{2}}.
\end{align}
By further upper bounding the right side of (\ref{ChainingIneq}), we obtain
\begin{align}
\ex[X_W]&\leq 3\sqrt{2}\sum_{k=k_0+1}^{\infty}2^{-k}\left(I(\pi_{\mathcal{N}_k}(W),\pi_{\mathcal{N}_{k-1}}(W);X_{\mathcal{N}_k}-X_{\mathcal{N}_{k-1}})\right)^{\frac{1}{2}}\nonumber\\
&\leq3\sqrt{2}\sum_{k=k_0+1}^{\infty}2^{-k}\left(I(\pi_{\mathcal{N}_k}(W),\pi_{\mathcal{N}_{k-1}}(W);X_{\mathcal{N}_k}-X_{\mathcal{N}_{k-1}},X_{\mathcal{N}_{k-1}})\right)^{\frac{1}{2}}\label{Follows from chain rule 1}\\
&=3\sqrt{2}\sum_{k=k_0+1}^{\infty}2^{-k}\left(I(\pi_{\mathcal{N}_k}(W),\pi_{\mathcal{N}_{k-1}}(W);X_{\mathcal{N}_k\cup \mathcal{N}_{k-1}})\right)^{\frac{1}{2}}\label{Follows from one to one}\\
&\leq 3\sqrt{2}\sum_{k=k_0+1}^{\infty}2^{-k}\left(I(\pi_{\mathcal{N}_k}(W),\pi_{\mathcal{N}_{k-1}}(W);X_{T})\right)^{\frac{1}{2}},\label{Follows from chain rule 2}
\end{align}
where (\ref{Follows from chain rule 1}) and (\ref{Follows from chain rule 2}) follow from the chain rule of mutual information, and (\ref{Follows from one to one}) follows from the fact that mutual information is invariant to one-to-one functions. 

\item From (\ref{The chaining sum}) we conclude that
\begin{equation}\label{The chaining sum abs value}
|X_W-X_{t_0}|\leq \sum_{k=k_0+1}^n |X_{\pi_{\mathcal{N}_k}(W)}-X_{\pi_{\mathcal{N}_{k-1}}(W)}|+|X_W-X_{\pi_{\mathcal{N}_n}(W)}|.
\end{equation}
Hence
\begin{equation}\label{The chaining sum abs value 2}
\ex [|X_W-X_{t_0}|] -\ex[|X_W-X_{\pi_{\mathcal{N}_n}(W)}|] \leq \sum_{k=k_0+1}^n \ex[|X_{\pi_{\mathcal{N}_k}(W)}-X_{\pi_{\mathcal{N}_{k-1}}(W)}|].
\end{equation}
The rest of the proof is similar to previous part, with the difference of instead of using Theorem \ref{FiniteExpectationMI} to obtain (\ref{Difference between part a and b}), we use Proposition \ref{AbsoluteValueFiniteCorollary} (b) with $\psi(\lambda)\triangleq \frac{\lambda^2\sigma^2}{2}$ to obtain
\begin{align}
&\ex \left[|X_{\pi_{\mathcal{N}_k}(W)}-X_{\pi_{\mathcal{N}_{k-1}}(W)}|\right]\nonumber\\
&\hspace{5mm}\leq 3\sqrt{2}\times 2^{-k}\left(I(\pi_{\mathcal{N}_k}(W),\pi_{\mathcal{N}_{k-1}}(W);\{X_{\mathcal{N}_k(t)}-X_{\mathcal{N}_{k-1}(t)}\}_{t\in T})+\log 2\right)^{\frac{1}{2}}.
\end{align}
\end{enumerate}
\end{proof}

\begin{remark}\label{ShannonEntropy}\normalfont
Note that for all $k>k_0$,
\begin{align}
I(\pi_{\mathcal{N}_k}(W),\pi_{\mathcal{N}_{k-1}}(W);X_{T})&\leq H(\pi_{\mathcal{N}_k}(W),\pi_{\mathcal{N}_{k-1}}(W))\\
																					    &\leq H\left(\pi_{\mathcal{N}_k}(W)\right)+H\left(\pi_{\mathcal{N}_{k-1}}(W)\right)\\
																					    &\leq \log |\mathcal{N}_k|+\log |\mathcal{N}_{k-1}|\\
																					    &\leq 2\log |\mathcal{N}_k|.
\end{align}
Therefore, if we assume that for each $k>k_0$, $\mathcal{N}_k$ is a \emph{minimal} $2^{-k}$-net for $(T,d)$, then we have replaced the Hartley entropy in Dudley's inequality (Theorem \ref{DudleyTheorem2}) with Shannon entropy (because $\log |\mathcal{N}_k|=\log N(T, d, 2^{-k})$) and further with mutual information.
\end{remark}
We are now able to present the proof of the small subset property theorem:
\begin{proof}[Proof of Theoerem \ref{PartitionTheorem}]
For each $k\geq k_1(T)$, let $\mathcal{N}^{(1)}_k$ and $\mathcal{N}^{(2)}_k$ be minimal $2^{-k}$-nets for $T_1$ and $T_2$, respectively. It is clear that $\mathcal{N}_k\triangleq \mathcal{N}^{(1)}_k\cup \mathcal{N}^{(2)}_k$, is a $2^{-k}$-net for $T$. Let
\[ \pi_{\mathcal{N}_k}(t)\triangleq\begin{cases} 
      \pi_{\mathcal{N}^{(1)}_k}(t) &  \mathrm{ if }~ t\in T_1 \\
      \pi_{\mathcal{N}^{(2)}_k}(t) & \mathrm{ if }~ t\in T_2 
   \end{cases}
\cdot\]
Based on Theorem \ref{ChainingIneqMITheorem} and Remark \ref{ShannonEntropy}, we have
\begin{align}
\ex[X_W]&\leq 3\sqrt{2}\sum_{k=k_1(T)}^{\infty}2^{-k}\left(H(\pi_{\mathcal{N}_k}(W))+H(\pi_{\mathcal{N}_{k-1}}(W)\right)^{\frac{1}{2}}\nonumber\\
&\leq 3\sqrt{2}\sum_{k=k_1(T)}^{\infty}2^{-k}\left(\alpha \log|\mathcal{N}^{(1)}_k|+(1-\alpha)\log|\mathcal{N}^{(2)}_k|\right.\nonumber\\
&\hspace{30mm}\left.+\alpha \log|\mathcal{N}^{(1)}_{k-1}|+(1-\alpha)\log|\mathcal{N}^{(2)}_{k-1}|+2H(\alpha)\right)^{\frac{1}{2}}\nonumber\\
&\leq 3\sqrt{2}\sum_{k=k_1(T)}^{\infty}2^{-k}\left(\alpha \log |\mathcal{N}^{(1)}_k|^2+(1-\alpha)\log |\mathcal{N}^{(2)}_k|^2+2H(\alpha)\right)^{\frac{1}{2}}\nonumber\\
&\leq 6 \sum_{k=k_1(T)}^{\infty}2^{-k}\left(\alpha \log |\mathcal{N}^{(1)}_k|+(1-\alpha)\log |\mathcal{N}^{(2)}_k|+H(\alpha)\right)^{\frac{1}{2}}\nonumber\\ 
&= 6\sum_{k=k_1(T)}^{\infty}2^{-k}\left(\alpha\log N(T_1,d,2^{-k})+(1-\alpha)\log N(T_2,d,2^{-k})+H(\alpha)\right)^{\frac{1}{2}}.
\end{align}
\end{proof}

For random processes other than subgaussian processes, where the tail of increments are controlled by a function $\psi$, we have the following result whose proof is similar to the proof of Theorem \ref{ChainingIneqMITheorem}:

\begin{theorem}\label{Chaining Ineq Psi Theorem}
Assume that $\{X_t\}_{t\in T}$ is a separable process defined on the bounded metric space $(T,d)$, with $\ex [X_t]=0$ for all $t\in T$ and
\begin{equation}
\log \ex\left[e^{\frac{\lambda(X_t-X_s)}{d(t,s)}}\right]\leq \psi(\lambda) \textrm{~for all~} t,s\in T, \lambda\geq 0,
\end{equation}
where $\psi$ is convex and $\psi(0)=\psi'(0)=0$. Let $k_0$ be an integer such that $2^{-k_0}\geq \mathrm{diam}(T)$ and $\{\mathcal{N}_k\}_{k=k_0+1}^{\infty}$ be a sequence of sets, where for each $k>k_0$, $\mathcal{N}_k$ is a $2^{-k}$-net for $(T,d)$. For an arbitrary $t_0\in T$, let $\mathcal{N}_{k_0}\triangleq\{t_0\}$. Assume that $W$ is a random variable which takes values on $T$. We have
\begin{enumerate}[(a)]
\item 
\begin{equation}
\ex [X_W] \leq 3\sqrt{2}\sum_{k=k_0+1}^{\infty}2^{-k}\ps \left(I(\pi_{\mathcal{N}_k}(W),\pi_{\mathcal{N}_{k-1}}(W);X_{T})\right).
\end{equation}
\item 
\begin{equation}
\ex \left[|X_W-X_{t_0}|\right] \leq 3\sqrt{2}\sum_{k=k_0+1}^{\infty}2^{-k}\ps \left(I(\pi_{\mathcal{N}_k}(W),\pi_{\mathcal{N}_{k-1}}(W);X_{T})+\log 2\right).
\end{equation}
\end{enumerate}
\end{theorem}
\end{appendix}
\end{document}